\documentclass[twoside]{article}

\usepackage[accepted]{aistats2023}

\usepackage[USenglish]{babel}
\usepackage{csquotes}

\usepackage{subcaption}
\usepackage{graphicx}
\usepackage{booktabs}
\usepackage{multirow}
\usepackage{amsmath,amssymb,amsthm}

\usepackage[disable]{todonotes}

\usepackage{hyperref}

\usepackage[backend=biber,style=authoryear,maxcitenames=2,maxbibnames=10,sortcites,sorting=ynt,uniquelist=minyear]{biblatex}
\let\citep\parencite
\let\citet\textcite
\renewbibmacro{in:}{}
\defbibheading{bibliography}[References]{\subsubsection*{#1}}

\usepackage[capitalize,noabbrev]{cleveref}

\newtheorem{prop}{Proposition}

\DeclareMathOperator*{\argmin}{argmin}
\DeclareMathOperator*{\E}{\mathbb{E}}
\DeclareMathOperator{\MMD}{MMD}
\DeclareMathOperator{\MMDusq}{\widehat{MMD}_U^2}
\DeclareMathOperator{\MMDBsq}{\widehat{MMD}_B^2}

\DeclareMathOperator{\N}{\mathcal{N}}
\newcommand{\ud}{\mathrm{d}}

\newcommand{\PP}{\mathbb{P}}
\newcommand{\QQ}{\mathbb{Q}}
\newcommand{\R}{\mathbb{R}}
\newcommand{\nullhyp}{\mathfrak{H}_0}
\newcommand{\althyp}{\mathfrak{H}_1}

\newcommand{\f}[1]{\textbf{\textcolor{red}{#1}}}
\newcommand{\s}[1]{\textcolor{blue}{#1}}
\newcommand{\T}[1]{\textcolor{green!40!black}{#1}}

\begin{document}

%

%

\twocolumn[

\aistatstitle{MMD-B-Fair: Learning Fair Representations with Statistical Testing}

\aistatsauthor{Namrata Deka \And Danica J.\ Sutherland}

\aistatsaddress{%
\texttt{dnamrata@cs.ubc.ca}\\%
University of British Columbia
\And%
\texttt{dsuth@cs.ubc.ca}\\%
University of British Columbia \& Amii} ]

\begin{abstract}
We introduce a method, MMD-B-Fair, to learn fair representations of data via kernel two-sample testing. We find neural features of our data where a maximum mean discrepancy (MMD) test cannot distinguish between representations of different sensitive groups, while preserving information about the target attributes. Minimizing the power of an MMD test is more difficult than maximizing it (as done in previous work), because the test threshold's complex behavior cannot be simply ignored. Our method exploits the simple asymptotics of block testing schemes to efficiently find fair representations without requiring complex adversarial optimization or generative modelling schemes widely used by existing work on fair representation learning. We evaluate our approach on various datasets, showing its ability to ``hide'' information about sensitive attributes, and its effectiveness in downstream transfer tasks.
\end{abstract}

\section{INTRODUCTION}
Machine learning systems are increasingly being used for making critical and sensitive real-life decisions in domains like finance, criminal reform, hiring, health, etc.~\citep{Flores2016FalsePF, Skeem2016RISKRA, Bogen2018HelpWA, Chouldechova2018ACS, Lebovits2018AutomatingIH, Ledford2019MillionsOB, Wilson2019PredictiveII} The importance of designing non-discriminatory learning algorithms that can mitigate various biases regarding private and protected features like gender or race is crucial to building trustworthy AI systems. Often data collected from the real world are plagued with issues like under-representation of minority groups, correlated sensitive and target features, or drastic distributional shifts between training and testing phases~\citep{gianfrancesco2018potential, jo2020lessons}. All of these can lead to biased models that can make undesirable mistakes in the real world, and therefore we need to address this issue and develop systems that are robust to biases in data distributions.

Fair representation learning is one approach towards this goal, which tries to find data representations that satisfy certain fairness objectives~\citep{Zemel2013LearningFR,Edwards2016CensoringRW,Louizos2016TheVF,Zhang2018MitigatingUB,Madras2018LearningAF,Lahoti2020FairnessWD}. Most deep learning-based fair representation learning methods take one of two broad approaches: try to disentangle latent factors with a generative variational model then ultimately discard the sensitive factor from the representation, or mitigate bias via adversarial techniques where discriminator(s) attempt to predict the sensitive group from a learnt encoded representation. In this work, we explore a different route, using deep kernels and statistical two-sample testing.

Statistical two-sample tests are used to determine whether two sets of data samples come from the same underlying distribution. Our method is centered around the idea that if a machine learning system is fair with respect to certain protected attributes, then that system’s representation of one sensitive group should not be statistically distinguishable from the other. Our method learns fair representations by optimizing a neural network to minimize the test power -- the ability of a two-sample test to correctly distinguish two sets of samples -- for samples differing by the sensitive class label, while still finding a useful representation by maximizing the test power and/or classification accuracy for distinguishing ``target'' labels.

This framework avoids learning a generative model of the data or explicit adversarial training, by instead relying on tests based on the maximum mean discrepancy (MMD) \citep{gretton2012kernel} to compare different samples of representations. We use the MMD in a novel way, combining existing work on power optimization \citep{Sutherland2017GenerativeMA,liu2020learning} with block testing \citep{Zaremba2013BtestsLV} to give an effective criterion for driving down the test power of sensitive tests -- a problem not handled well by previous work which focuses only on maximizing power. Our method is supported by theoretical results as well as good empirical performance.

We first give a self-contained introduction to MMD-based testing in \cref{sec:prelim}, establishing all the tools we will need for our method for learning fair kernels and representations (\cref{sec:approach}), and emphasizing aspects important to our approach.

\section{PRELIMINARIES}
\label{sec:prelim}
Based on \textit{i.i.d.}\ samples $S_{\PP}$ and $S_{\QQ}$ from distributions $\PP$ and $\QQ$, respectively, the two-sample testing problem asks whether $S_{\PP}, S_{\QQ}$ come from the same distribution: does $\PP = \QQ$? We use the null hypothesis testing framework, i.e.\ ask whether we can confidently say that the observed $S_\PP$ and $S_\QQ$ would be unlikely to be so different if $\PP = \QQ$.

Traditional methods for two-sample tests, including $t$-tests and Kolmogorov-Smirnov tests,
do not scale to complex high-dimensional distributions. Another modern approach is based on classification accuracy and we will describe our approach's relationship to that scheme shortly.

\subsection{MAXIMUM MEAN DISCREPANCY (MMD)}
The MMD \citep{gretton2012kernel} is a measure of distance between distributions. For distributions $\PP$ and $\QQ$ over a domain $\mathcal X$ (the set of conceivable data points), the MMD is defined in terms of a kernel $k: \mathcal X \times \mathcal X \rightarrow \R$ giving the ``similarity'' of individual data points. This kernel should be positive semi-definite, the simplest case being the linear kernel $k(x, y) = x^\top y$, and the paradigmatic example being a Gaussian kernel $k(x, y) = \exp(-\tfrac{1}{2\sigma^2} \lVert x - y \rVert^2)$.

If $X, X' \sim \PP$ and $Y, Y' \sim \QQ$, then
\[
\MMD(\PP, \QQ; k) = \sqrt{\E[k(X,X')+k(Y,Y')-2k(X,Y)]}.
\]
With a \emph{characteristic} kernel $k$, such as the Gaussian, we have that $\MMD(\PP, \QQ; k)=0$ \textit{if and only if} $\PP = \QQ$. Thus, we can run a two-sample test by estimating the MMD, and rejecting the null hypothesis that $\PP = \QQ$ if the estimated MMD is too large to have occurred by chance.

\paragraph{$U$-STATISTIC ESTIMATOR}
Our default estimator will be the $U$-statistic estimator, which is unbiased for $\MMD^2$, and has almost minimal variance among unbiased estimators:\footnote{The MVUE would simply also include the $k(X_i, Y_i)$ terms; the difference in practice is usually trivial, but this form is slightly simpler and allows exact expressions for the variance.}
\begin{gather}
  \MMDusq(S_{\PP}, S_{\QQ}; k) = \frac{1}{m(m-1)}\sum_{i\neq j}H_{ij}  \label{eq:mmdusq} \\
  H_{ij} = k(X_i,X_j) + k(Y_i,Y_j) - k(X_i,Y_j) - k(Y_i,X_j) \notag
,\end{gather}
where $S_{\PP} = \{X_1, \dots, X_m\}, S_{\QQ} = \{Y_1, \dots, Y_m\}$ are \textit{i.i.d.} samples from $\PP$ and $\QQ$ respectively.

The most common scheme for testing based on \eqref{eq:mmdusq} is to choose some kernel $k$ a-priori, and then reject the null hypothesis $\nullhyp$ that $\PP = \QQ$
if the scaled estimator $m \MMDusq(S_\PP, S_\QQ; k)$ is larger than a threshold $c_\alpha$.
The rejection threshold, $c_\alpha$, should satisfy
$\Pr_{\nullhyp}\left(m \MMDusq(S_\PP, S_\QQ; k) > c_\alpha\right) \le \alpha$,
i.e.\ there is $\alpha$ probability of incorrectly rejecting $\nullhyp$ when it is true.
The estimate is scaled by $m$ because, as $m$ grows, $m \MMDusq(S_\PP, S_\QQ; k)$ converges in distribution to an infinite mixture of $\chi^2$ variables, with weights depending on $\PP = \QQ$ and $k$, but independent of $m$.
The rejection threshold $c_\alpha$ is the $(1-\alpha)$th quantile of the distribution over $m \MMDusq(S_\PP, S_\QQ; k)$ under $\nullhyp$ which can be approximated
with a scheme known as permutation testing, generally the preferred method in this case:
randomly divide $S_\PP \cup S_\QQ$ into two groups, compute $m \MMDusq$ between them and repeat, taking the empirical quantile of those samples \citep{Sutherland2017GenerativeMA}.

\paragraph{BLOCK ESTIMATOR}
An alternative approach,
called B-testing by \citet{Zaremba2013BtestsLV}, randomly splits the available samples into $b$ blocks each containing $B$ samples. This is more computationally efficient in its estimator and also allows avoiding permutation testing, as we will see shortly.
We compute $\MMDusq$ on each block separately and since each of those terms will be an independent unbiased estimator of the squared MMD, we can average them to obtaining the block-based estimator $\MMDBsq$.

Under $\nullhyp$, the estimate in each block converges in distribution to the kernel-dependent infinite mixture of $\chi^2$ variables as $B \to \infty$.
However, whether under $\nullhyp$ or $\althyp$,
the average of $b$ of these independent estimates will converge to a normal distribution by the central limit theorem:
\begin{equation}
    \sqrt{b} ( \MMDBsq - \MMD^2 ) \xrightarrow{d} \mathcal N(0, V_B)
    \label{eq:block-asymp}
,\end{equation}
with $V_B$ being the variance of $\MMDusq$ on samples of size $B$ (depending on $\PP$, $\QQ$, and $k$).
A block test, then, can take as its test statistic $\sqrt{b} \MMDBsq$
and use a threshold of $\sqrt{V_B} \, \Phi^{-1}(1 - \alpha)$,
with $\Phi$ the CDF of a standard normal.

To use this method,
it remains to estimate $\sqrt{V_B}$.
\citet{Zaremba2013BtestsLV} simply took the sample standard deviation of the $b$ batches, which is justified since the sample variance converges almost surely to $V_B$.
We will employ a different scheme in our use of the block estimator (to come).
Although block tests are more computationally efficient than $U$-statistic tests,
it turns out they are also proportionally less powerful \citep{ramdas2015adaptivity} and therefore, our primary tests will be based on $U$-statistics.

\subsection{LEARNING DEEP KERNELS}
MMD tests work well when the choice of kernel $k$ is appropriate; for complicated distributions, however, simple default choices may take unreasonable numbers of samples to obtain a significant power. For a powerful test in complex situations with realistic numbers of samples, we follow \citet{liu2020learning} in seeking the best kernel from a parameterized family of \emph{deep kernels}. Specifically, we take $k_{\omega}$ as a Gaussian kernel $\kappa$ on the output of a featurizer network $\phi_{\omega}$, $k_{\omega} = \kappa_\omega(\phi_{\omega}(x),\phi_{\omega}(y))$.
Here, $\phi_\omega$ is a deep neural network that extracts features from input points $x$ and $y$, whose parameters are contained within $\omega$, and $\kappa_\omega$ is a Gaussian kernel on those features whose length-scale is also contained in $\omega$. These kernels have seen success across a variety of areas \citep[e.g.][]{andrew:deep-kernels,mmdgans,jean:semisup-deep-kernels,li:ssl-hsic}.

To be able to reliably distinguish two distributions, we wish to find the deep kernel with the most powerful test: the one with the highest probability of correctly rejecting the null hypothesis when the alternative is true. For a $U$-statistic test, this probability is asymptotically
\begin{equation} \label{eq:power}
    {\Pr}_{\althyp}\left( m \MMDusq > c_{\alpha} \right)
    \rightarrow \Phi\left( \frac{\MMD^2 - c_\alpha / m }{\sqrt{V_m}} \right),
\end{equation}
where $\Phi$ is the CDF of a standard normal distribution,
and $V_m$ is the variance of the $\MMDusq$ estimator for samples of size $m$ from $\PP$ and $\QQ$ with the kernel $k$ \citep[Equation 2]{Sutherland2017GenerativeMA}. The terms on the right-hand side are fixed, unknown quantities depending on $\PP$, $\QQ$, and $k$; $\MMD^2$ and $c_\alpha$ do not depend on $m$. This formula comes from an asymptotic normality result for the estimator when $\MMD(\PP, \QQ; k) > 0$ \citep[Section 5.5]{Serfling1980ApproximationTO}.

\citet{Sutherland2017GenerativeMA,liu2020learning} conducted tests by dividing each of $S_\PP$ and $S_\QQ$ into ``training'' and ``test'' sets, finding a kernel approximately maximizing \eqref{eq:power} on the training sets, and then using that kernel to run a standard two-sample test on the independent test sets. To roughly maximize \eqref{eq:power}, they maximized an estimator of $\MMD^2 / \sqrt{V_m}$, the leading term when $m$ grows and the test is reasonably likely to reject ($m \MMD^2 > c_\alpha$).

Although this was not done in prior work, it will be important for our purposes to emphasize that \eqref{eq:power} is the asymptotic expression for the power of a test using $m$ samples, and so a given $k$, $\PP$, and $\QQ$ correspond to a whole curve of asymptotic powers depending on $m$. Inside \eqref{eq:power}, both $\MMD^2$ and $c_\alpha$ are independent of $m$, while, as we will see, $V_m$'s dependence on $m$ is exactly known thanks to the well-understood theory of $U$-statistics. Thus, we can estimate the power of an $m$-sample test using a \emph{different} number of samples $n$. For instance, we could get a rough estimate of the power of a large-sample test ($m = 2,000$) using a small mini-batch of size $n = 32$.

To roughly maximize \eqref{eq:power}, \citet{liu2020learning} maximized the estimator $\MMDusq / \sqrt{\widehat{V}_{m,\lambda}}$,
where $\widehat{V}_{m,\lambda}$ estimates $V_m$ by
\begin{equation} \label{eq:var-est}
\!\!\!
\frac{4}{m n^3} \sum_{i=1}^n \left( \sum_{j=1}^n H_{ij} \right)^2
- \frac{4}{m n^4} \left( \sum_{i=1}^n \sum_{j=1}^n H_{ij} \right)^2
+ \frac \lambda m
,\end{equation}
using $H_{ij}$ from \eqref{eq:mmdusq}.
For \citeauthor{liu2020learning}'s purposes, $m$ is a simple scalar multiplier on the objective and so need not be specified, but it will be  important for us to keep track of it, as we will see. They further proved uniform convergence of the estimator $\MMDusq / \sqrt{\widehat{V}_{m,\lambda}}$ to $\MMD^2 / \sqrt{V_m}$. \citet{Sutherland2017GenerativeMA} used a more complex unbiased estimator for $V_m$ \citep[see][]{sutherland:unbiased-mmd-variance}; an unbiased estimator for $V_m$ will not be unbiased for $\MMD^2 / \sqrt{V_m}$, however, and in fact we prove in \cref{app:no-unbiased} that \emph{no} unbiased estimator of that quantity exists. The biased estimator also worked better in our experiments.

\citet{Sutherland2017GenerativeMA} further mentioned, but did not try, using the threshold from permutation testing to estimate the full quantity \eqref{eq:power}; this is expected to be important for small $m$ or for tests with poor power (ignoring the $c_\alpha$ term means the overall asymptotic power cannot be less than $0.5$). This estimator, as an empirical quantile, is almost surely differentiable and straightforward to implement in deep learning libraries. We explore this further in \cref{sec:approach}.

As argued by \citet[Section 4]{liu2020learning}, learning a deep kernel for an MMD test is strictly more general than classifier two-sample tests \citep{c2st-orig,c2st}, which train a classifier between $\PP$ and $\QQ$ on the training split, then check whether it has nontrivial accuracy on the test split. The added generality tends to yield better tests in practice.

\section{LEARNING FAIR REPRESENTATIONS}\label{sec:approach}
Let $\PP^a$ and $\QQ^a$ be conditional distributions on a dataset that only differ by the value of the binary feature $a$ on which they condition: e.g.\ $\PP^a$ is the distribution of data where $a = 0$, and $\QQ^a$ the distribution where $a = 1$. Take corresponding sample sets $S_{\PP^a}$, $S_{\QQ^a}$. In this section we will outline our approach for learning either a fair kernel or a fair vector representation.

We will assume in this paper that the relevant attributes $a$ have two possible values, but extensions to a small number of discrete values are straightforward.

\subsection{LEARNING A FAIR KERNEL}
Our goal is to find a representation invariant with respect to a binary sensitive attribute $s$, meaning that it cannot distinguish $\PP^s$ and $\QQ^s$: the distribution of data points with $s = 0$ and those with $s = 1$. To achieve this, we would like to find a kernel which, when used in a two-sample test to distinguish $\PP^s$ and $\QQ^s$, achieves negligible power.

If this were our only goal, however, there is a trivial solution: use, say, $k(x, y) = 1$.
Instead, we would like a kernel that is also useful to distinguish \emph{target} pairs of distributions, say ones useful for a downstream task: one that has high test power between $\PP^t$ and $\QQ^t$. (In practice, we also include a classification loss in our objective, but we clarify this straightforward addition later.)

One simple extension to the objective function of \citet{liu2020learning} towards this goal would be to minimize an estimate of $\left((\MMD^t)^2 / \sqrt{V^t_m} - (\MMD^s)^2 / \sqrt{V^s_m}\right)$, where $(\MMD^a)^2$ and $V^a_m$ are computed for the learned kernel between $\PP^a$ and $\QQ^a$. However, this tends to be unable to appropriately ``balance'' the two objectives. If the power for the target test is near $1$, but the sensitive-attribute test still has high power, this objective would still be just as satisfied by driving up $(\MMD^t)^2 / \sqrt{V^t_m}$ -- increasing the asymptotic power of the target test, but only just barely -- as it would be by reducing $(\MMD^s)^2 / \sqrt{V^s_m})$.

To put the two attributes on the same scale, then, we should consider the full asymptotic power \eqref{eq:power}, and subtract estimators of the two,
resulting in the objective:
\begin{equation}\label{eq:perm-power}
\!\!\!\!\!\!
\resizebox{0.90\hsize}{!}{%
    $\Phi\left( \frac{(\MMD^t)^2 - c^t_\alpha / m }{\sqrt{V^t_m}} \right) - \Phi\left( \frac{(\MMD^s)^2 - c^s_\alpha / m }{\sqrt{V^s_m}} \right)$%
}
.\end{equation}
The thresholds $c^s_\alpha$ and $c^t_\alpha$, can be estimated using permutation tests as suggested by \citet{Sutherland2017GenerativeMA}. This makes the optimization substantially more computationally expensive; though it can be computed based on the same kernel matrix as $\MMDusq$ and $\widehat{V}_m$, it requires perhaps a hundred times as many matrix-vector multiplications as does $\MMDusq$. We also found that the strong dependence between $\hat c_\alpha$ and $\MMDusq$ computed on the same samples meant that optimization was rarely able to drive the asymptotic power for the sensitive attribute test below about $0.5$. Data splitting helped, but halves the effective batch size, and computational and sample complexity both suffer.

To avoid this problem, we instead optimize the power of a block test with $b$ blocks of size $B$. From the central limit result \eqref{eq:block-asymp}, we have that the power of a block test is, letting $t_\alpha = \Phi^{-1}(1 - \alpha)$ where $\Phi$ is the standard normal CDF,
\begin{align}
       \rho_{b,B}
  &  = \Pr_{\althyp}\Big( \sqrt b \, \MMDBsq > \sqrt{V_B} t_\alpha  \Big)   \notag
\\&  = \Pr_{\althyp}\left( \frac{\sqrt b \, (\MMDBsq - \MMD^2)}{\sqrt{V_B}} > t_\alpha - \frac{\sqrt b \MMD^2}{\sqrt{V_B}} \right)  \notag
\\&\to \Phi\left( \sqrt b \, \frac{\MMD^2}{\sqrt{V_B}} - t_\alpha \right)
\label{eq:block-power}
.\end{align}
The block test's constant asymptotic threshold gives us a simple form that is cheaper to compute than using the permutation test threshold in \eqref{eq:power}, is valid even for small values of the population power, and only uses the samples in the form of the ratio $\MMD^2 / \sqrt{V_B}$ - which we already know can estimated effectively \citep{liu2020learning}.
We can thus estimate the asymptotic power with
\begin{equation}
    \hat\rho_{b,B}
    = \Phi\left( \sqrt b \, \frac{\MMDusq}{\sqrt{\widehat V_{B,\lambda}}} - t_\alpha \right)
    \label{eq:block-power-est}
.\end{equation}
$\hat\rho_{b,B}$ will converge uniformly to $\rho_{b,B}$ over classes of deep kernels satisfying some technical assumptions
as a corollary of \citet{liu2020learning};
proof in \cref{app:convergence}.

Using \eqref{eq:block-power-est}, our objective to learn a fair kernel with sensitive attribute $s$ and target attribute $t$ is
\begin{equation}\label{eq:fair-kernel}
    \argmin_\omega \left[
        \hat\rho_{b,B}^s - \hat\rho_{b,B}^t
    \right]
.\end{equation}

Although we are optimizing a kernel based on the power $\rho_{b,B}$ of a block test, we do not use blocking in our estimator; we just find a more amenable objective based on the asymptotic power of a hypothetical block test -- closely related to power of the $U$-statistic test.

\subsection{LEARNING FAIR REPRESENTATIONS}
So far we have shown how to learn an optimal kernel that can simultaneously achieve high power for distinguishing target attributes, and low power for sensitive attributes.
If we wish to learn a feature \emph{representation} rather than a single kernel,
however, it is not enough that a \emph{particular} kernel cannot distinguish the sensitive attribute; we would ideally like that \emph{no} usage of that representation with any kernel can distinguish between $\PP^s$ and $\QQ^s$, while maintaining that at least one kernel can distinguish between $\PP^t$ and $\QQ^t$. That is, if we separate into a representation $\phi$ and a kernel $\kappa$ on that representation, we would like to solve
\begin{equation}\label{eq:minimax}
    \min_\phi \left[ \max_\kappa \hat\rho_{b,B}^s - \max_\kappa \hat\rho_{b,B}^t \right]
.\end{equation}

The objective \eqref{eq:minimax} could be optimized with an alternating minimax optimization scheme for the parameters of $\kappa$, looking something like an MMD-GAN \citep{mmdgans,binkowski:mmd-gans}. We find it sufficient in our experiments to use a much simpler scheme: a grid of Gaussian kernels of varying length-scales. This finds a fairer kernel than using a single Gaussian, preventing the representation $\phi$ from learning to just ``hide'' information at a very different scale than the single $\kappa$ examines, while being much simpler to implement and optimize than in alternating gradient schemes for GAN like models.

\subsection{CONDITIONAL POWER FOR STRONG CORRELATIONS}\label{sec:eq-version}
So far in our discussion, the two-sample tests are based on the distributions $\PP^s = \PP_{X \mid S = 0}$ and $\QQ^s = \QQ_{X \mid S = 1}$. This setting learns a representation that optimizes the demographic parity (\texttt{DP}), defined as \[ \texttt{DP} = 1 - \lvert P(\hat{T} = 1 \mid S = 0) - P(\hat{T} = 1 \mid S = 1) \rvert .\]
In our approach, this setting has the advantage of not requiring both target and sensitive labels simultaneously for any data point in the training set, i.e., it still works if we have separate collections of data points labeled for the target and for the sensitive attribute.
Moreover, it works even if we do not have a high-confidence labeling of the sensitive attribute, but instead have rough estimates collected e.g.\ via randomized response methods \citep{warner:randomized-response}.
The DP setting, however, struggles when the target and sensitive attributes are strongly correlated so that the sample pairs $(S_{\PP^t}, S_{\QQ^t})$ and $(S_{\PP^s}, S_{\QQ^s})$ come from very similar pairs of distributions. This makes the objective of minimizing the test power over one pair while maximizing the test power over the other very difficult.

To address this, we instead condition the sensitive pair over the target classes, and sample points from $\PP^{s \mid t} =\PP_{X \mid S = 0, T = t}$ and $\QQ^{s \mid t} = \QQ_{X \mid S = 1, T = t}$ for all values of $T$. This is now equivalent to maximizing for the equalized odds (\texttt{EO}) notion of fairness with respect to all distinct target classes $t$, defined as
\begin{multline*}
\texttt{EO} = 1 - \big\lvert
P(\hat{T} = t \mid T = t, S = 0)
\\ - P(\hat{T} = t \mid T = t, S = 1) \big\rvert
.\end{multline*}
This modifies the sensitive power objectives in \eqref{eq:fair-kernel} and \eqref{eq:minimax} to, summing over the possible values of $t$,
\begin{gather}
    \argmin_\omega \left[
        \left(\sum_t \hat\rho_{b,B}^{s \mid t}\right) - \hat\rho_{b,B}^t
    \right]
    ,\\
    \min_\phi \left[\max_\kappa \left(\sum_t \hat\rho_{b,B}^{s \mid t}\right) - \max_\kappa \hat\rho_{b,B}^t \right] .
    \label{eq:eq-objs}
\end{gather}

It is well-known that perfect demographic parity, $\texttt{DP} = 1$, is not generally compatible with perfectly equalized odds, $\texttt{EO} = 1$ \citep{Barocas2018FairnessAM}.
Even so, Theorem 3.1 of \citet{Zhao2020ConditionalLO} shows that classifiers satisfying $\texttt{EO} = 1$ have demographic parity gaps $\Delta_{\texttt{DP}}$ upper-bounded by the gap of a perfect classifier, and hence training with an equalized odds criterion does not strongly compromise demographic parity.

\subsection{ADDING A CLASSIFIER TASK LOSS}
Representations with strong power on a target task are likely able to strongly distinguish at least some portion of samples as belonging to a certain value of $t$.
If our final goal is to train a classifier, though,
it will help to try to ensure our representation can classify all points well, by adding a standard classification loss for $t$ to our objectives,
e.g.\
\begin{equation*}
    \min_{\phi,g}
        \left[ \max_\kappa \lambda_s\left(\sum_t \hat\rho_{b,B}^{s \mid t}\right) - \max_\kappa \lambda_t\hat\rho_{b,B}^t 
        + \lambda_{\text{cls}} L^t(g \circ \phi) \right]
,\end{equation*}
where $g$ is a classifier on $\phi$,
$L(g \circ \phi, t)$ is the cross-entropy loss of the classifier $g(\phi(x))$ with labels $t$,\footnote{For the equalized-odds objective, we evaluate the classification loss on all samples. For the demographic parity version, we only evaluate it on the points from $S_{\PP^t}$ and $S_{\QQ^t}$, to ensure the method does not require any samples with both $s$ and $t$ values.}{} and $\lambda_s$, $\lambda_t$, $\lambda_{\text{cls}}$ control the relative regularization strengths. We perform an ablation study showing the significance of the classifier loss in \cref{sec:experiments}.

\section{RELATED WORK}\label{sec:related}
Fair representation learning has of late (deservedly) found a lot of traction within the deep learning community \citep{Mehrabi2021ASO}. The growing popularity and success of adversarial learning has resulted in a substantial number of adversarial techniques to mitigate bias and enforce group fairness by training discriminators to distinguish one sensitive group (or sub-group) from another~\citep{Edwards2016CensoringRW, Xie2017ControllableIT, Zhang2018MitigatingUB, Madras2018LearningAF, Zhao2020ConditionalLO}. However, representations learnt via adversarial approaches do not completely ``hide" sensitive information as the learnt representations are dependent on the specific function classes (or architectural complexity) used for the discriminators. Variational methods, on the other hand, focus on learning disentangled latent spaces where sensitive factors can be separated from non-sensitive features~\citep{Louizos2016TheVF, Creager2019FlexiblyFR, Norouzi2020VariationalFI}. Other methods (including our proposed approach) try to enforce fairness by adding additional constraints in the learning objective to regularize the learned weights of the neural networks involved~\citep{Kamishima2012ConsiderationsOF,Hajian2016AlgorithmicBF,Zafar2017FairnessBD,Speicher2018AUA}.

There have also been, in particular, several MMD-based approaches to fair/invariant representation learning. \citet{Louizos2016TheVF} used the MMD as a regularizer to train fair variational autoencoders to impose statistical parity between latent embeddings across different sensitive groups. Recently, \citet{Oneto2020ExploitingMA} used the MMD with a similar intuition to ours to learn representations that transfer better to unseen tasks in a multitask setting. \citet{Veitch2021CounterfactualIT} use the MMD as regularizers to a classifier, choosing between the marginal and conditional form based on the causal direction of the task, to enforce counterfactual invariance. Most recently \citet{lee2022fairpca} proposed using the MMD to perform fair principal component analysis by penalizing the measure between dimensionality-reduced distributions over different protected groups. Our approach, although similar in spirit, uses the power of MMD two-sample tests rather than the raw MMD estimate, which avoids several pitfalls and is particularly important when simultaneous maximization and minimization are required -- something not previously explored in the kernel-methods community.

In \cref{sec:experiments}, we compare to several different baselines.
LAFTR~\citep{Madras2018LearningAF} employs an adversarial network to predict the sensitive class using the representations being simultaneously learnt by a target predictor.
CFAIR~\citep{Zhao2020ConditionalLO} conditionally aligns the representations for accuracy-fairness trade-off by using two adversaries (one for the positive class label and one for the negative label).
FCRL~\citep{Gupta2021ControllableGF} controls the mutual information between the representations and the sensitive labels with contrastive information estimators.
sIPM~\citep{Kim2022LearningFR} employs the sigmoid Integral Probability Metric (IPM) as the deviance measure over the learnt representations. This is perhaps the most closely related method to our approach of using an IPM measure to regularize the prediction function.

\section{EXPERIMENTS}\label{sec:experiments}
We evaluate both versions, \eqref{eq:minimax} and \eqref{eq:eq-objs}, of our proposed regularizer -- we call these MMD-B-Fair (DP) and MMD-B-Fair (Eq) -- against the baselines sIPM~\citep{Kim2022LearningFR}, FCRL~\citep{Gupta2021ControllableGF}, CFAIR~\citep{Zhao2020ConditionalLO} and LAFTR~\citep{Madras2018LearningAF}.
One testbed is the widely used UCI Adult dataset~\citep{Dua:2019} -- a structured dataset to predict whether an individual has income above \$50,000 USD while being fair to their gender. We also evaluate performance on COMPAS\footnote{\url{github.com/propublica/compas-analysis}} which contains criminal records of over 5000 people living in Florida. The task is to predict recidivism (binary) within the next two years while being sensitive to the race of an individual (also binary).
The final dataset we evaluate on is the Heritage Health\footnote{\url{foreverdata.org/1015/}} dataset, which contains records of insurance claims and physician information of over 60,000 patients. The primary task is to predict Charlson index - an estimate  of the risk of a patient's death over the next ten years - without being biased by the age at which they first claimed an insurance cover.

We present results of fairness-accuracy trade-offs and various downstream tasks along with an ablation study to investigate the importance of all of the terms in our loss function. Our code is available at \href{https://github.com/namratadeka/mmd-b-fair}{\nolinkurl{github.com/namratadeka/mmd-b-fair}}.
\begin{table}[h]
\centering
\resizebox{\columnwidth}{!}{%
\begin{tabular}{@{}c|c|ccc@{}}
\toprule
\textbf{Dataset}                 & \textbf{} & \textbf{Train} & \textbf{Val} & \textbf{Test} \\ \midrule
\multirow{2}{*}{Adult}           & $\chi^2$  & 1177.9         & 238.5        & 0.0           \\
                                 & p-value   & 3.96e-258      & 8.33e-54     & 1.0           \\ \midrule
\multirow{2}{*}{COMPAS}          & $\chi^2$  & 26.032         & 5.263        & 20.944        \\
                                 & p-value   & 3.35e-07       & 0.021        & 4.72e-06      \\ \midrule
\multirow{2}{*}{Heritage Health} & $\chi^2$  & 6565.2         & 1606.9       & 8260.8        \\
                                 & p-value   & 0              & 0            & 0             \\ \bottomrule
\end{tabular}%
}
\caption{$\chi^2$-test of independence between target and sensitive variables in the data.}
\label{tab:data-chi}
\end{table}

\paragraph{EXPERIMENTAL SETUP}
We train all the algorithms across different choices of their respective fairness hyper-parameters. For both versions of our method we set $\lambda_s$ to $\{0, 0.1, 1, 10, 100, 1000, 10000\}$ with a fixed $\lambda_t$ and $\lambda_{\text{cls}}$ of $1$. For sIPM, CFAIR and LAFTR we set the regularization strength to the same set of values as $\lambda_s$, and for FCRL we use a subset of the hyper-parameters ($\beta$ and $\lambda$) proposed in their paper. We train all models with a mini-batch size of 64 and report the average performance over ten independent seeds. Wherever possible, the encoder architecture is shared across different methods. More details about the training process can be found in \cref{app:training}.

We perform a $\chi^2$-test of independence between the sensitive and target attributes to better understand the performance over each dataset. The test statistics and respective p-values within each split is shown in \cref{tab:data-chi}. In the Adult dataset there is a co-variate shift between the train and test domains where the target and sensitive variables goes from being strongly dependent in the train set to being completely independent in the test set.
\paragraph{FAIRNESS}
\begin{figure}[h]
    \centering
    \includegraphics[width=\columnwidth]{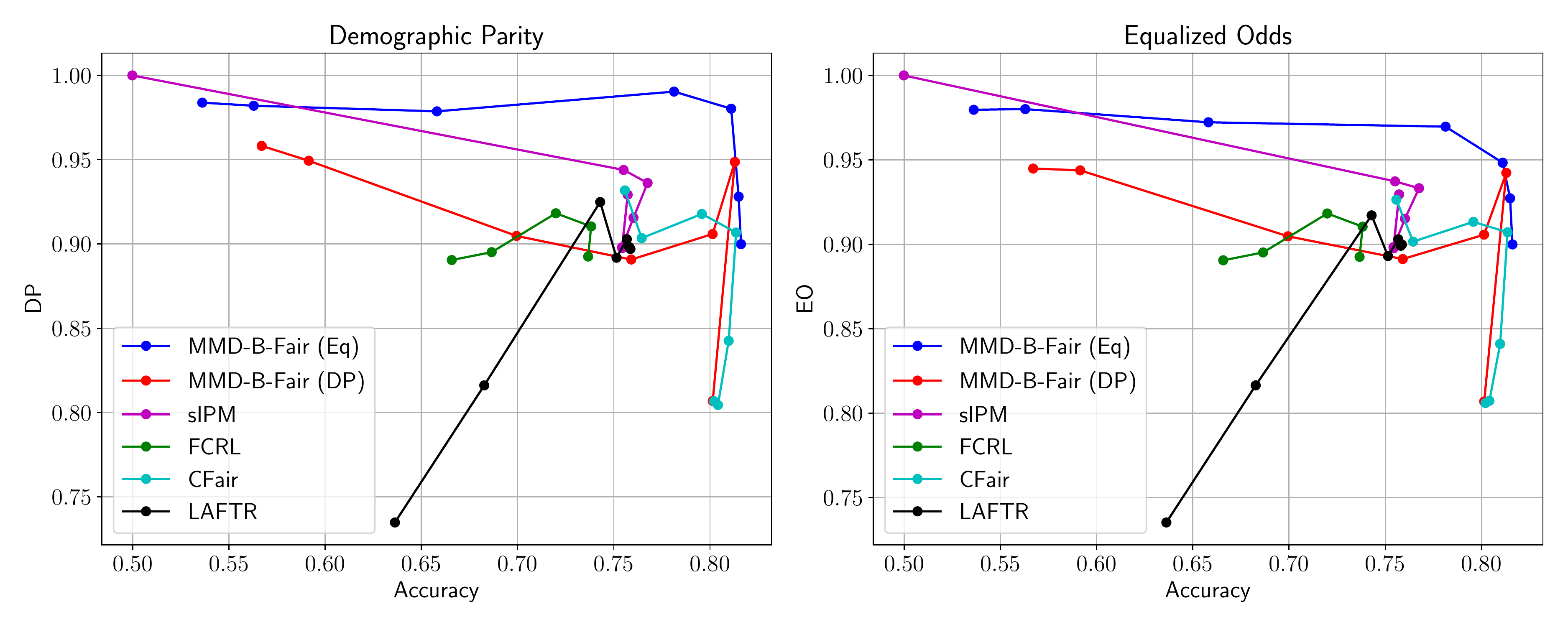}
    \includegraphics[width=\columnwidth] {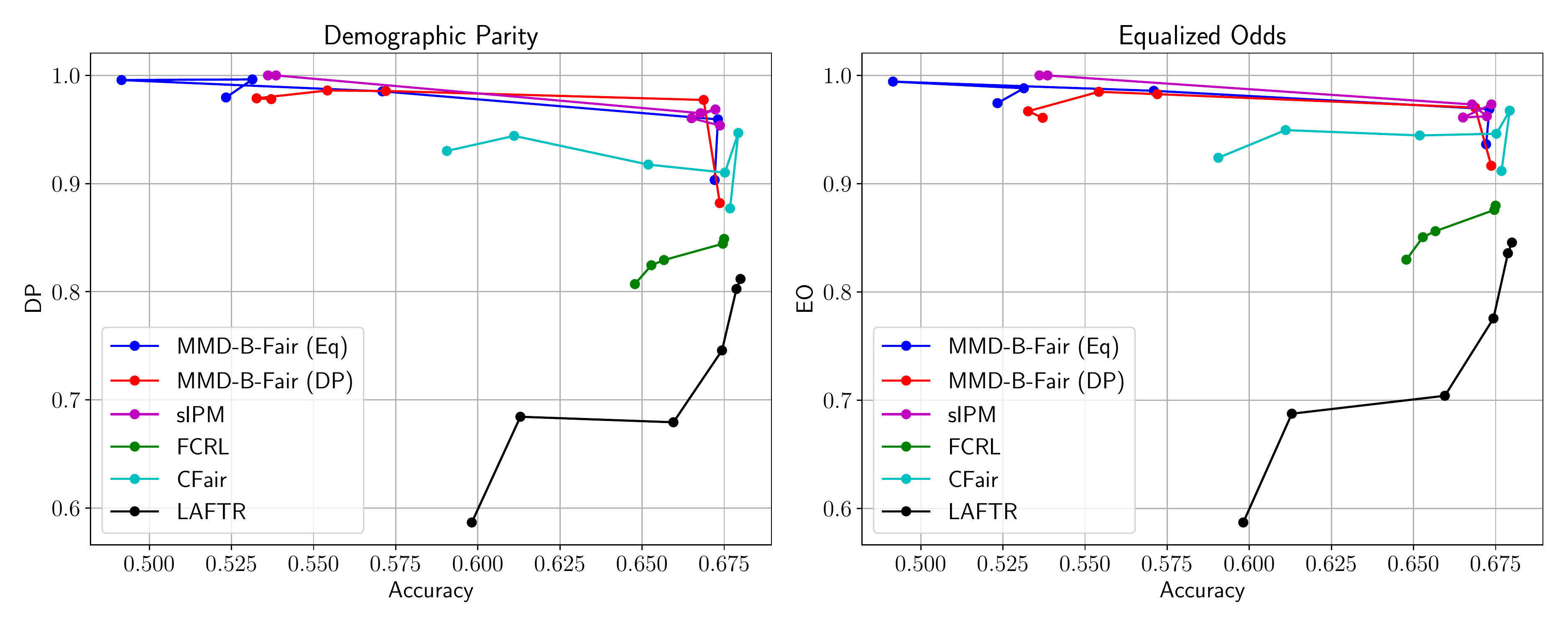}
    \includegraphics[width=\columnwidth] {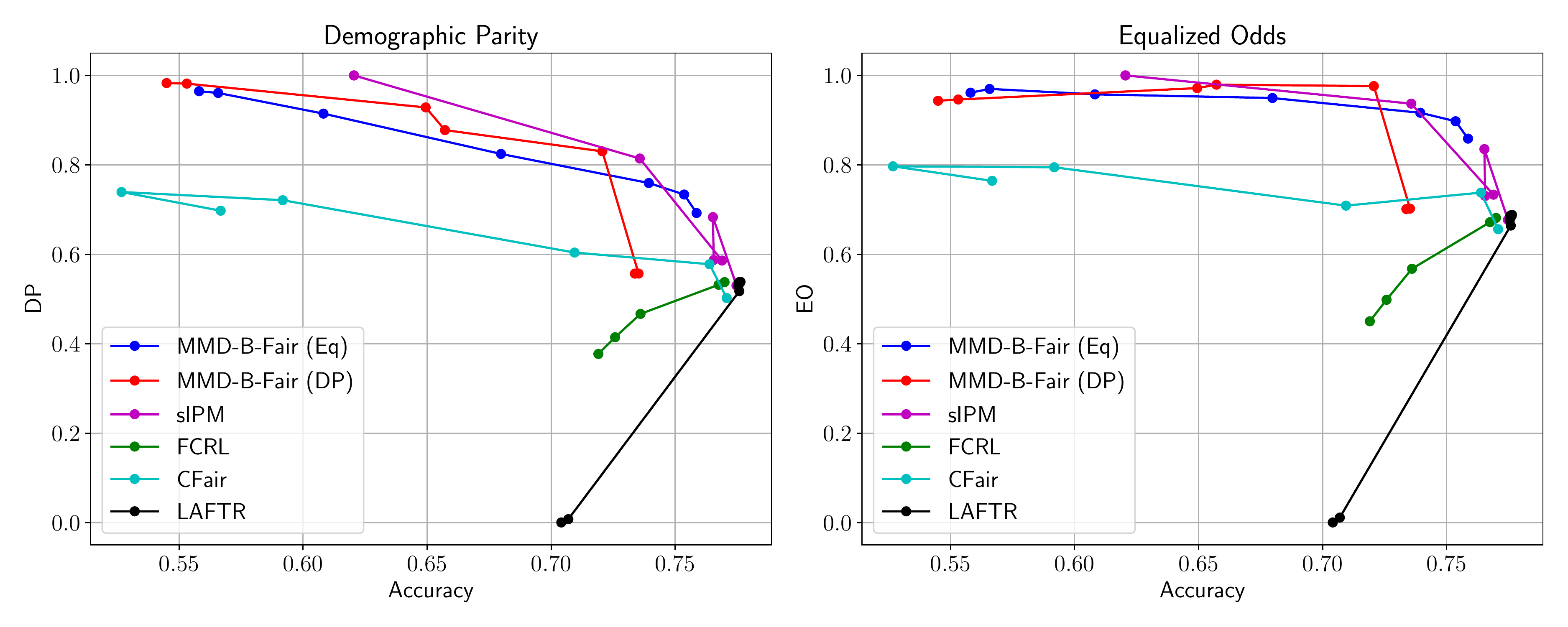}
    \caption{Fairness-accuracy trade-off curves on the test set of (top) Adult, (middle) COMPAS and (bottom) Heritage Health. Higher values for all metrics are better.}
    \label{fig:tradeoff}
\end{figure}
Firstly, we examine the fairness-accuracy tradeoff fronts obtained by sweeping over the fairness hyper-parameters in \cref{fig:tradeoff}.
The $x$-axis is the target accuracy; the $y$-axis reports the Demographic Parity (\texttt{DP}) and Equalized Odds (\texttt{EO}), averaged over both positive and negative target classes. Note that higher values are better.

For the Adult dataset (\cref{fig:tradeoff}, top), MMD-B-Fair (Eq) outperforms the baselines, concurrently achieving high accuracy scores and fairness measures. Recall the co-variate shift across the train and test split in this dataset further highlighting the robustness of our method compared to others.
In the absence of co-variate shift across splits, both of our methods and sIPM perform equally well on the COMPAS (\cref{fig:tradeoff}, middle) and Heritage Health (\cref{fig:tradeoff}, bottom) datasets.

\begin{figure}
    \centering
    \includegraphics[width=\columnwidth] {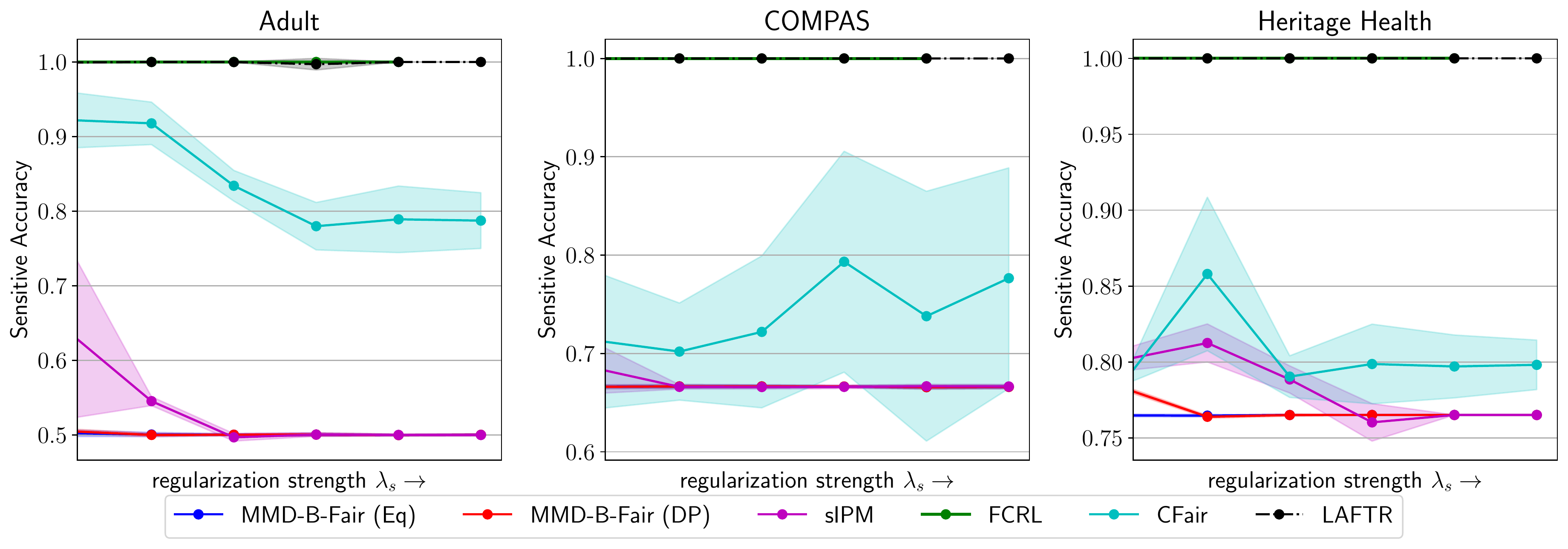}
    \caption{Downstream sensitive label classification over fair representations. Majority class probabilities: Adult: $0.5$, COMPAS: $0.66$, Heritage Health: $0.76$.}
    \label{fig:scls}
\end{figure}
\paragraph{EXAMINING LEARNT REPRESENTATIONS}
A popular method for evaluating fair models is to examine if the learnt representations contain enough information to predict the sensitive labels: if all information regarding the sensitive attributes is successfully hidden in the representation learning phase, then subsequent classifiers will struggle to discriminate between the sensitive classes and learn to assign the majority class label to each sample to maximize the classification accuracy. This accuracy will be equal to the probability of the majority class label. On the test set, these probabilities are $0.5$ for Adult, $0.66$ for COMPAS and $0.76$ for Heritage Health. We train MLP classifiers over the learnt representations, and show in \cref{fig:scls} the sensitive classification performance as a function of the fairness regularization strengths used to train the underlying fair models.
\begin{figure}
    \centering
    \includegraphics[width=\columnwidth] {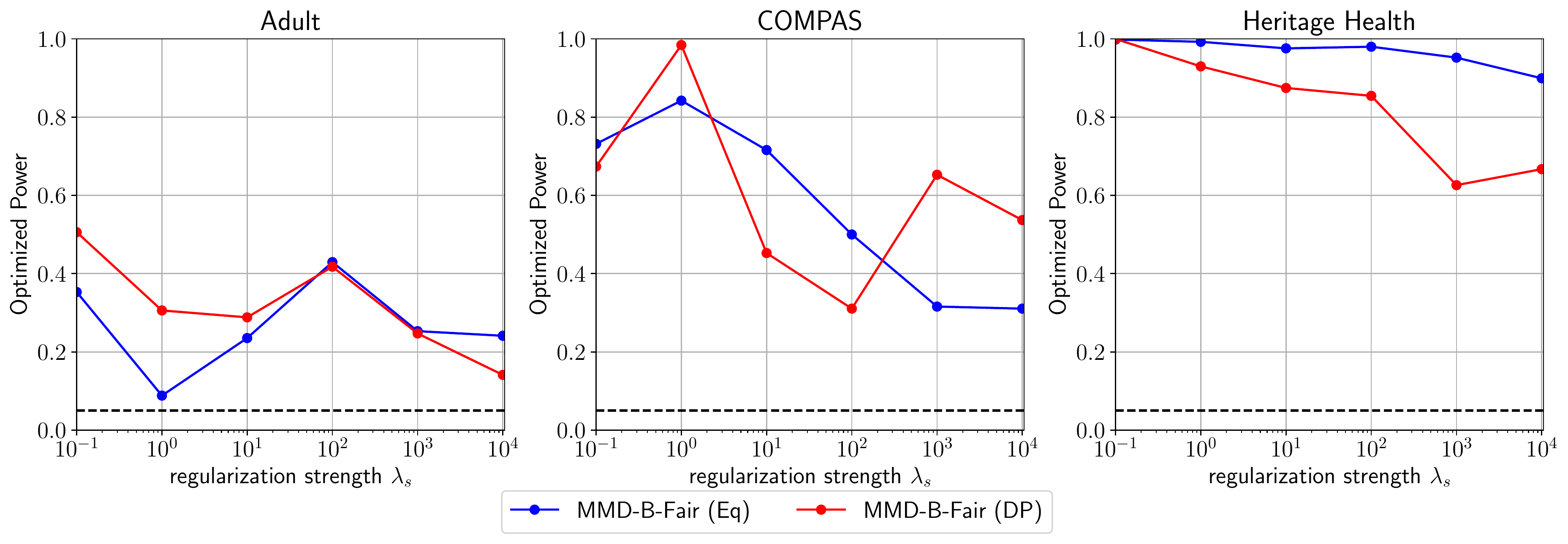}
    \caption{Empirical test power with an optimized kernel to maximize sensitive power over learnt representations. }
    \label{fig:skernel}
\end{figure}

\begin{figure*}[t]
  \centering
  \includegraphics[width=0.16\linewidth]{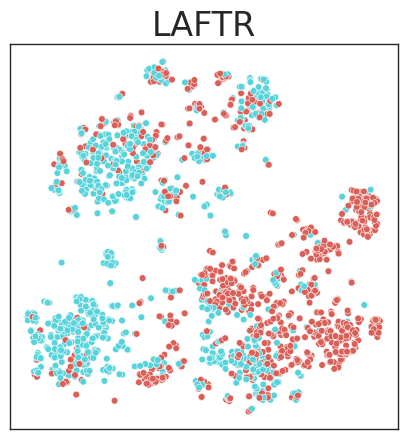}
  \includegraphics[width=0.16\linewidth]{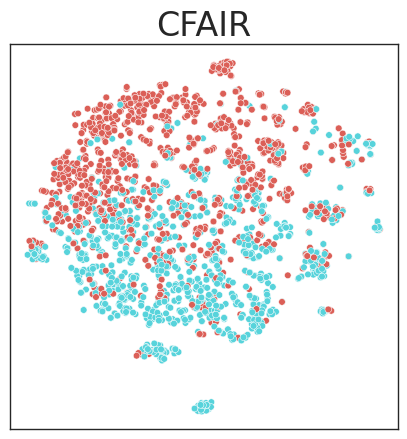}
  \includegraphics[width=0.16\linewidth]{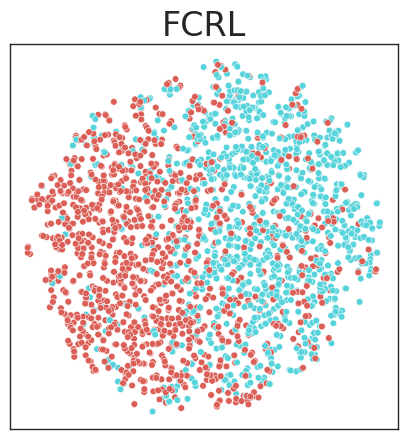}
  \includegraphics[width=0.16\linewidth]{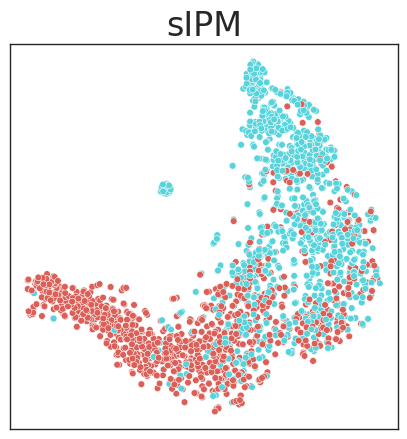}
  \includegraphics[width=0.16\linewidth]{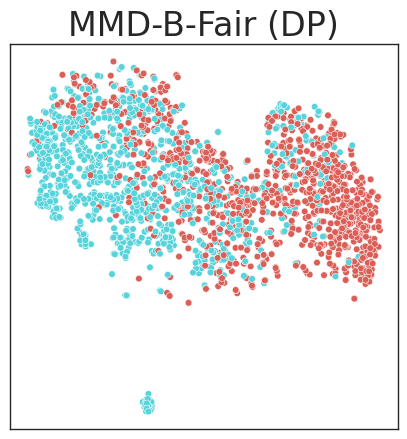}
  \includegraphics[width=0.16\linewidth]{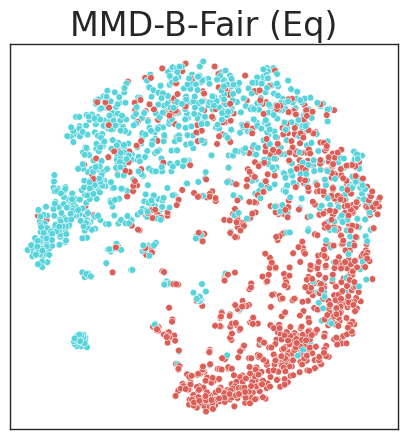}
  \\
  \includegraphics[width=0.16\linewidth]{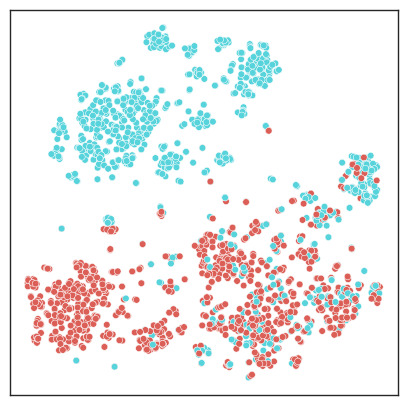}
  \includegraphics[width=0.16\linewidth]{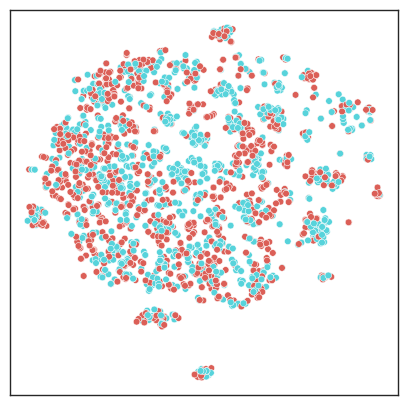}
  \includegraphics[width=0.16\linewidth]{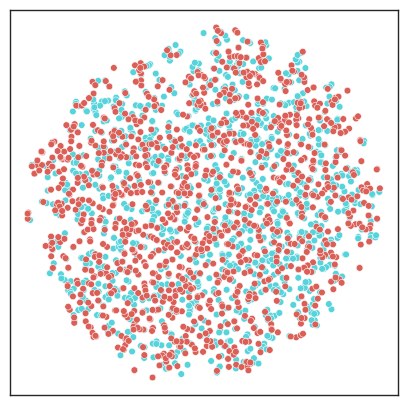}
  \includegraphics[width=0.16\linewidth]{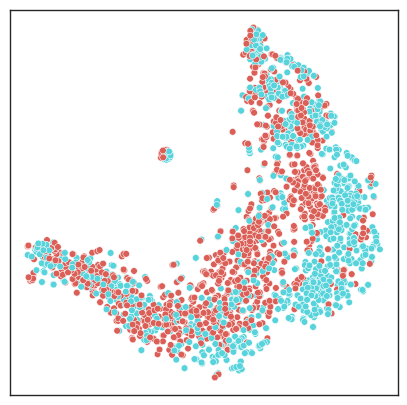}
  \includegraphics[width=0.16\linewidth]{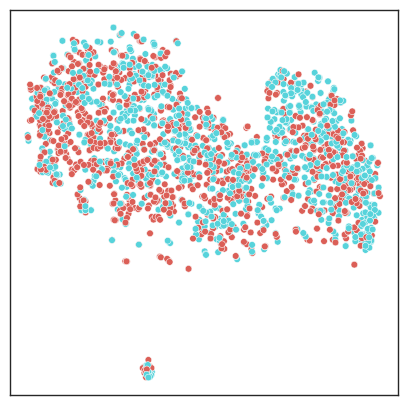}
  \includegraphics[width=0.16\linewidth]{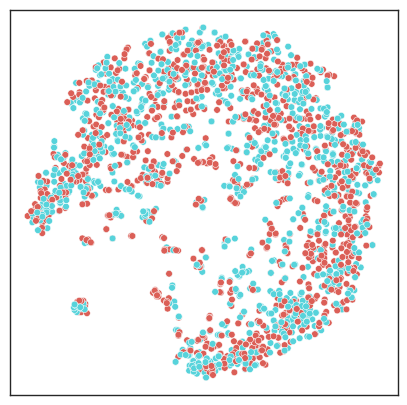}
  
  \caption{t-SNE visualizations of Adult representations, colored by target attribute (top) and sensitive attribute (bottom).}
  \label{fig:tsne}
\end{figure*}
Both versions of our method, as well as sIPM, are able to maintain the desired accuracy score equal to the fraction of the majority sensitive label in the test set for each dataset. sIPM converges to the ideal accuracy at sightly higher regularization strengths compared to MMD-B-Fair, while classifiers over representations from FCRL and LAFTR easily achieve perfect sensitive accuracy scores of 100\% even with strong regularization indicating their failure to be invariant to sensitive information.

Checking the sensitive accuracy is essentially a classifier-based two-sample test \citep{c2st} between $\PP^s$ and $\QQ^s$ based on the learnt representations.
We also try using a more sensitive measure of whether these representations are the same:
the power of an MMD two-sample test with a learned kernel, which is more general and often more powerful than a classifier-based test \citep{liu2020learning}. For models with classification accuracies significantly above random performance, this power will be near-perfect as it might be that even if few individual points can be correctly classified, a two-sample test will be able to distinguish the distributions as a whole.
We run this check for our methods in \cref{fig:skernel}, using a Gaussian kernel with a learnt length-scale over a one-layer MLP architecture trained to roughly maximize the asymptotic power $\hat\rho_{b,B}^s$ operating on top of the fair representations as input. We then evaluate the empirical power of this test i.e., how many times it rejects the null hypothesis, while repeating the test with 64 samples at a time. As expected, two-sample tests are far more sensitive measures of attribute leakage than classification accuracy.
\begin{table}[h]
\centering
\resizebox{\columnwidth}{!}{%
\begin{tabular}{@{}c|c|c|c|c|c|c|c@{}}
\toprule
Transfer Label &  & LAFTR & CFAIR & FCRL & sIPM & \begin{tabular}[c]{@{}c@{}}MMD-B-Fair \\ (DP)\end{tabular} & \begin{tabular}[c]{@{}c@{}}MMD-B-Fair \\ (Eq)\end{tabular} \\ \midrule
          & acc &  {57.2} &  {62.5} &  {58.0}&\f{72.8}&\s{71.3}&\T{70.3}\\
MSC2a3    & DP  &  {52.3} &  {65.1} &\f{99.2}&  {69.3}&\T{72.2}&\s{84.5}\\
          & Eq  &  {57.4} &  {70.1} &\f{98.0}&  {69.9}&\T{71.8}&\s{86.6}\\ \midrule
          & acc &\f{72.9} &\s{72.2} &  {53.9}&\T{72.4}&  {70.7}&  {69.4}\\
METAB3    & DP  &  {52.3} &\T{65.1} &\f{97.7}&  {54.5}&  {65.6}&\s{82.1}\\
          & Eq  &  {61.3} &\T{77.1} &\f{97.6}&  {63.4}&  {74.6}&\s{92.1}\\ \midrule
          & acc &  {66.4} &  {65.9} &  {59.3}&\f{70.6}&\T{67.5}&\s{67.8}\\
ARTHSPHIN & DP  &  {52.3} &  {65.1} &\f{98.0}&  {74.6}&\T{83.0}&\s{87.7}\\
          & Eq  &  {54.9} &  {70.1} &\f{98.1}&  {76.7}&\T{84.9}&\s{90.0}\\ \midrule
          & acc &  {64.4} &  {61.9} &  {60.1}&\f{68.0}&\T{67.1}&\s{67.3}\\
NEUMENT   & DP  &  {52.3} &  {65.1} &\f{99.1}&  {72.9}&\T{86.8}&\s{94.5}\\
          & Eq  &  {54.9} &  {69.7} &\f{97.5}&  {73.2}&\T{86.7}&\s{95.4}\\ \midrule
          & acc &  {71.0} &  {67.3} &  {69.3}&\f{73.5}&\s{73.0}&\T{72.5}\\
MISCHRT   & DP  &  {52.3} &  {65.1} &\f{98.6}&  {85.0}&\T{87.2}&\s{96.4}\\
          & Eq  &  {59.4} &  {79.0} &\f{98.2}&  {88.5}&\T{88.6}&\s{97.5}\\ \bottomrule
\end{tabular}%
}
\caption{Using Heritage Health representations to predict various downstream tasks. \f{Red} marks the best result per row, \s{blue} second-best, and \T{green} third-best.}
\label{tab:health-transfer}
\end{table}

\Cref{fig:tsne} shows $t$-SNE visualizations of learnt latent space embeddings,
further demonstrating that our method's representations separate the target attribute well and make the sensitive attribute difficult to distinguish.

\paragraph{FAIR TRANSFER LEARNING}
A major goal of fair \emph{representation} learning, rather than simply finding a fair classifier, is to be able to use the same representations for more than one potential downstream task. We would like our representations to have good (and fair) performance for classifiers when trained on tasks unknown at the representation learning time, even for downstream classifiers that are trained without any concern about fairness at all: the representations should enforce it.

To model this situation, we take representations learned to predict Charlson Index on Heritage Health and use them to predict each of five Primary Condition Groups, which were left out in the original representation learning phase. We train these classifiers without regard to fairness by simply minimizing the cross-entropy loss.

\cref{tab:health-transfer} shows the resulting accuracy scores with respect to the transfer labels and fairness scores with respect to the original sensitive labels of downstream classifiers trained on each representation. With these representations, MMD-B-Fair (Eq) provides stronger fairness results than any competitor except FCRL (which is quite inaccurate), while being more accurate than any competitor except sIPM (which is quite unfair).

\paragraph{ABLATION STUDY}
\begin{figure}
    \centering
    \includegraphics[width=\columnwidth] {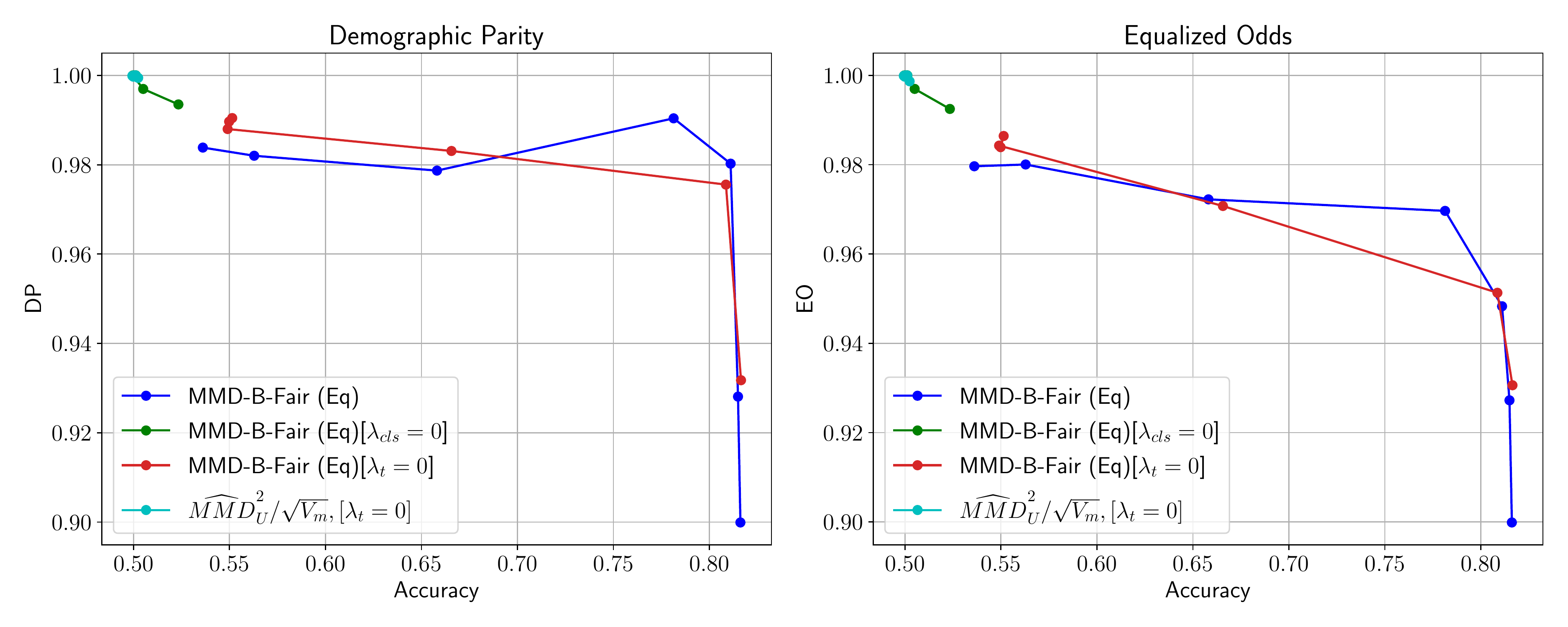}
    \caption{Assessing the contribution of each loss term on the performance on Adult.}
    \label{fig:ablation}
\end{figure}
Since our objective consists of three terms - target classification loss, sensitive power and target power - we perform an ablation study in \cref{fig:ablation} to ascertain the contribution of each term to learning fair representations that can achieve high target accuracy. When the classification loss is turned off by setting $\lambda_{cls}$ to $0$, we see from the tradeoff curve that a downstream classifier trained on top of the learnt representations fail to achieve a good accuracy score. Turning off the target power instead (by setting $\lambda_t = 0$) does not have this effect, however the fairness metrics are slightly impacted at the high accuracy regime. Supposing this is not a significant drop in fairness measures, we also train a model that directly minimizes the normalized sensitive MMD instead of the power (which, recall from our discussion in \cref{sec:approach} was used to balance both sensitive and target terms when used together). However, in this case we observe that the MMD measure by itself overwhelms the classifier leading to representations that are perfectly fair but come at the cost of random target classification performance.

\section{CONCLUSION}\label{sec:outro}
We proposed a method for learning fair kernels as well as representations built off of two-sample testing -- a different paradigm than previous approaches to learning fair representations. Our approach combines two-sample techniques in a novel way by using the $U$-statistic estimator to estimate the power of a block test which may also be useful for other testing approaches where one may need to minimize a test power.

We provide two different versions of our approach -- the  \texttt{DP} (demographic parity) version which can be trained using weak set-level labels from disjoint datasets, albeit at a disadvantage when dealing with correlated features, and a conditional (equalized odds) version, which can handle correlation between features well. Our method performs well compared to previous approaches based on adversarial learning and generative modelling when the dependency between the target and sensitive attributes is not the same in the train and test sets, i.e., when the i.i.d. assumption is violated. Downstream tasks like fair transfer learning also achieve a better balance between fairness and accuracy when using our learnt representations.

Areas for future work include extending to continuous-valued sensitive attributes via the Hilbert-Schmidt Independence Criterion \citep{hsic} and exploring applications in domain adaptation, invariant feature learning, causal representation learning, etc.

\subsubsection*{ACKNOWLEDGEMENTS}
This work was supported in part by the Natural Sciences and Engineering Resource Council of Canada (NSERC), the Canada CIFAR AI Chairs program, WestGrid, SHARCNET, Calcul Qu\'ebec, and the Digital Resource Alliance of Canada.
We would also like to particularly thank an anonymous reviewer for pointing out a flaw in our framing of a previous version of the algorithm.

\newrefcontext[sorting=nyt]
\printbibliography

\clearpage
\onecolumn
\appendix

\section{Non-existence of an unbiased estimator} \label{app:no-unbiased}

\begin{prop}
For any fixed kernel $k$,
let
$J(\PP, \QQ) = \MMD^2(\PP, \QQ) / \sqrt{V_m(\PP, \QQ)}$
for some $m > 2$.
Let $\mathcal P$ be some class of distributions such that $\{ (1-\alpha) \PP_0 + \alpha \PP_1 : \alpha \in [0, 1] \} \subseteq \mathcal P$,
where $\PP_0 \ne \PP_1$ are two distributions
with $\MMD(\PP_0, \PP_1) > 0$.
Then no estimator of $J$ can be unbiased on $\mathcal P$.
\end{prop}
\begin{proof}
We follow \citet{binkowski:mmd-gans} in using the broad approach of \citet{bickel-lehmann}.
Let $\PP_\alpha = (1 - \alpha) \PP_0 + \alpha \PP_1$ denote a mixture between $\PP_0$ and $\PP_1$.

Suppose there is some unbiased estimator $\hat J(X, Y)$,
meaning that for some finite $n_1$ and $n_2$,
\[
\E_{\substack{X \sim \PP^{n_1}\\Y \sim \QQ^{n_2}}} \hat J(X, Y)
= J(\PP, \QQ)
.\]
Then,
for any fixed $\QQ \in \mathcal P$,
the function
\begin{align*}
     R(\alpha)
  &= J(\PP_\alpha, \QQ)
\\&= \int \cdots \int
    \hat J(X, Y)
    \, \ud\PP_\alpha(X_1) 
    \cdots
    \ud\PP_\alpha(X_{n_1}) 
    \, \ud \QQ^{n_2}(Y)
\\&= \int \cdots \int
    \hat J(X, Y)
    \, [(1-\alpha) \ud\PP_0(X_1) + \alpha \ud \PP_1(X_1)]
    \cdots
    \, \ud \QQ^{n_2}(Y)
\\&= (1 - \alpha)^{n_1} \E_{\substack{X \sim \PP_0^{n_1}\\Y \sim \QQ^{n_2}}}[\hat J(X, Y)]
   + \cdots
   + \alpha^{n_1} \E_{\substack{X \sim \PP_1^{n_1}\\Y \sim \QQ^{n_2}}}[\hat J(X, Y)]
\end{align*}
must be a polynomial in $\alpha$.

But, if we pick $\QQ = \PP_1$,
we will show that
\[
     R(\alpha)
   = \frac{\MMD^2(\PP_\alpha, \PP_1)}{\sqrt{V_m(\PP_\alpha, \PP_1)}}
\]
is not a polynomial,
and thus no unbiased estimator can exist on $\mathcal P$.

To do this, we will need some notation,
and some unfortunately tedious calculations.
Let
\begin{align*}
  \PP_\alpha &= (1-\alpha) \PP_0 + \alpha \PP_1 \\
  \mu_\alpha &= \E_{X \sim \PP_\alpha} k(X, \cdot) 
    = (1-\alpha) \mu_0 + \alpha \mu_1 \\
  C_\alpha &= \E_{X \sim \PP_\alpha} k(X, \cdot) \otimes k(X, \cdot)
    = (1-\alpha) C_0 + \alpha C_1
,\end{align*}
where $\mu_\alpha$ is the kernel mean embedding of $\PP_\alpha$,
and $C_\alpha$ its (uncentered) covariance operator.
Here $k(x, \cdot)$ is the embedding of the point $x$ into the RKHS corresponding to the kernel $k$,
satisfying $\langle k(x, \cdot), k(y, \cdot) \rangle = k(x, y)$,
and $a \otimes b$ is the outer product of two vectors in a Hilbert space, a linear operator such that $[a \otimes b] c = a \langle b, c \rangle$.

The numerator of $R(\alpha)$ is
\[
    \MMD(\PP_\alpha, \PP_1)^2
    = \lVert (1-\alpha) \mu_0 + \alpha \mu_1 - \mu_1 \rVert^2
    = (1 - \alpha)^2 \MMD(\PP_0, \PP_1)
.\]
The denominator is much more complex,
but equation (2) of \citet{sutherland:unbiased-mmd-variance} shows that
\begin{align*}
    V_m(\PP_\alpha, &\PP_1) = \frac{2}{m (m-1)} \Big[ 
\\&   2(m-2) \langle \mu_\alpha, C_\alpha \mu_\alpha \rangle
    - (2m-3) \lVert \mu_\alpha \rVert^2
\\&   2(m-2) \langle \mu_1, C_1 \mu_1 \rangle
    - (2m-3) \lVert \mu_1 \rVert^2
\\& + 2(m-2) \langle \mu_1, C_\alpha \mu_1 \rangle
    + 2(m-2) \langle \mu_\alpha, C_1 \mu_\alpha \rangle
    - 2(2m-3) \langle \mu_\alpha, \mu_1 \rangle^2
\\& - 4(m-1) \langle \mu_\alpha, (C_\alpha + C_1) \mu_1 \rangle
    + 4(m-1) \left( \lVert \mu_\alpha \rVert^2 + \lVert \mu_1 \rVert^2 \right) \langle \mu_\alpha, \mu_1 \rangle
\\& + \E_{(X, X') \sim \PP_\alpha^2} k(X, X')^2
    + \E_{(Y, Y') \sim \PP_1^2} k(Y, Y')^2
    + 2 \E_{X \sim \PP_\alpha, Y \sim \PP_1} k(X, Y)^2
    \Big].
\end{align*}
We need not give a full expansion of $V_m$ in terms of $\alpha$;
we will merely show that it is of degree three.
Since the ratio of a degree-two polynomial with the square root of a degree-three polynomial cannot possibly be itself polynomial,
that will suffice to show that $R(\alpha)$ is not polynomial,
and hence no unbiased estimator exists.

To see this, notice that $\mu_\alpha$ and $C_\alpha$ are each linear in $\alpha$,
so that any term containing fewer than three such terms, e.g.\ $\lVert \mu_\alpha \rVert^2$ or $\langle \mu_\alpha, C_1 \mu_\alpha \rangle$, cannot possibly be of degree three and so is not relevant to our goal.
The expectations of squared kernels are also not relevant:
the highest-order in terms of $\alpha$ is
\begin{multline*}
    \E_{X, X' \sim \PP_\alpha} k(X, X')^2
    = (1-\alpha)^2 \E_{X, X' \sim \PP_0} k(X, X')^2
    + 2 \alpha (1-\alpha) \E_{\substack{X \sim \PP_0\\X' \sim \PP_1}} k(X, X')^2
    + \alpha^2 \E_{X, X' \sim \PP_1} k(X, X')^2
\end{multline*}
which is $\mathcal O(\alpha^2)$, abusing notation slightly to mean ``terms of degree 2 or lower in $\alpha$.''
This leaves us
\begin{align*}
    V_m(\PP_\alpha, &\PP_1) = \frac{2}{m (m-1)} \Big[ 
      2(m-2) \langle \mu_\alpha, C_\alpha \mu_\alpha \rangle
     + 4(m-1) \lVert \mu_\alpha \rVert^2 \langle \mu_\alpha, \mu_1 \rangle
    \Big]
    + \mathcal O(\alpha^2).
\end{align*}
We can find the $\alpha^3$ terms by
\begin{align*}
     \langle \mu_\alpha, C_\alpha \mu_\alpha \rangle
  &= (1-\alpha) \langle \mu_\alpha, C_\alpha \mu_0 \rangle
   + \alpha \langle \mu_\alpha, C_\alpha \mu_1 \rangle
\\&= \alpha \langle \mu_\alpha, C_\alpha (\mu_1 - \mu_0) \rangle + \mathcal O(\alpha^2)
\\&= \alpha^2 \langle \mu_\alpha, (C_1 - C_0) (\mu_1 - \mu_0) \rangle + \mathcal O(\alpha^2)
\\&= \alpha^3 \langle \mu_1 - \mu_0, (C_1 - C_0) (\mu_1 - \mu_0) \rangle + \mathcal O(\alpha^2)
\end{align*}
and
\begin{align*}
     \lVert \mu_\alpha \rVert^2 \langle \mu_\alpha, \mu_1 \rangle
  &= \alpha \langle \mu_\alpha, \mu_\alpha \rangle \langle \mu_1 - \mu_0, \mu_1 \rangle + \mathcal O(\alpha^2)
\\&= \alpha^2 \langle \mu_\alpha, \mu_1 - \mu_0 \rangle \langle \mu_1 - \mu_0, \mu_1 \rangle + \mathcal O(\alpha^2)
\\&= \alpha^3 \langle \mu_1 - \mu_0, \mu_1 - \mu_0 \rangle \langle \mu_1 - \mu_0, \mu_1 \rangle + \mathcal O(\alpha^2)
.\end{align*}
Because we assumed $\MMD(\PP_0, \PP_1) > 0$,
we have $\mu_1 \ne \mu_0$.
Thus these two terms cancel only if
\[
    \left\langle
        \mu_1 - \mu_0, 
        \left[
            (m-2) (C_1 - C_0)
            + 2(m-1) (\mu_1 - \mu_0) \otimes \mu_1
        \right]
        (\mu_1 - \mu_0)
    \right\rangle
    = 0
.\]
Now, suppose we had
defined $R(\alpha)$ with $\QQ = \PP_\beta$ rather than $\PP_1$ for some other $\beta \in [0, 1]$.
The only relevant thing that changes is that the lone $\mu_1$ above becomes $\mu_\beta$;
the numerator stays quadratic in $\alpha$.
Thus, if the terms cancel for $\mu_1$,
we can simply choose a different $\mu_\beta$
for which they do not cancel,
which will always be possible.
Thus the denominator is the square root of a degree-three polynomial,
$R(\alpha)$ is not a polynomial,
and no unbiased estimator can exist.
\end{proof}

\section{Uniform convergence of our objective} \label{app:convergence}
We show here that optimizing the approximated block-test power from \eqref{eq:block-power-est} with a finite number of samples from each conditional distribution works, i.e.\ as $m$ increases, our power estimate converges uniformly over the parameter space towards an optimal solution.

\citet{liu2020learning} proved that with probability at least $1 - \delta$ over the choice of $n$ samples used in the estimators
\begin{equation} \label{eq:existing-uc}
    \sup_{k \in \mathcal K}
    \left\lvert
    \frac{\MMDusq}{\sqrt{n \widehat V_{n, n \cdot n^{-1/3}}}}
    - \frac{\MMD^2}{\sqrt{\lim_{m \to \infty} m V_m}}
    \right\rvert
    \le \alpha(\mathcal K, \PP, \QQ, n, \delta)
\end{equation}
for some function $\alpha$
(given asymptotically in their Theorem 6 and Proposition 9, or with full constants in their Theorem 11 and Proposition 23; see also their Remarks 24 and 25).
Here $\mathcal K$ is the class of considered kernels;
note that $m V_m$ converges to a constant.

Notice from \eqref{eq:var-est} that,
for any $m$ and $\ell$,
$\widehat V_{\ell,\lambda} = \frac{m}{\ell} \widehat V_{m,\lambda}$.
Thus we can rewrite \eqref{eq:block-power-est} as
\[
    \hat \rho_{b,B}
    = \Phi\left(
        \sqrt{b} \frac{\MMDusq}{\sqrt{\widehat V_{B,\lambda}}}
        - t_\alpha
    \right)
    = \Phi\left(
        \sqrt{b B} \;
        \frac{\MMDusq}{\sqrt{n \widehat V_{n,\lambda}}}
        - t_\alpha
    \right)
    = \Phi\left(
        \sqrt{b B} \;
        \hat J_\lambda
        - t_\alpha
    \right)
,\]
where we defined
$\hat J_\lambda = \MMDusq / \sqrt{n \widehat V_{n,\lambda}}$.

Defining
$J = \MMD^2 / \sqrt{\lim_{m \to \infty} m V_m}$,
we can now rewrite
\eqref{eq:existing-uc}
more compactly as showing that, with probability at least $1 - \delta$,
$\sup_{k \in \mathcal K} \lvert \hat J_{n^{2/3}} - J \rvert \le \alpha(\mathcal K, \PP, \QQ, n, \delta)$.

Also, notice from \eqref{eq:block-power} that
$\rho_{b,B} \to \Phi(\sqrt{b B} J - t_\alpha) =: R_{b,B}$,
the asymptotic power of a test with $b$ blocks of size $B$.

Finally, the function
$x \mapsto \Phi(\sqrt{b B} x - t_\alpha)$
is Lipschitz continuous:
\begin{align*}
    \left\lvert \frac{\partial}{\partial x} \Phi(\sqrt{b B} x - t_\alpha) \right\rvert
    = \frac{1}{\sqrt{2 \pi}} \exp\left(
        - \frac{1}{2} \left( \sqrt{b B} x - t_\alpha \right)^2
    \right)
    \le \frac{1}{\sqrt{2 \pi}}
.\end{align*}
Thus applying this function to each of the terms in \eqref{eq:existing-uc} yields that,
when we use $\lambda = n^{2/3}$,
\[
    \sup_{k \in \mathcal K} \left\lvert
        \hat\rho_{b,B} - R_{b,B}
    \right\rvert
    \le \frac{1}{\sqrt{2 \pi}} \alpha(\mathcal K, \PP, \QQ, n, \delta)
.\]

This shows uniform convergence of each $\hat\rho_{b,B}$ to the relevant asymptotic power.
By a union bound, this immediately implies uniform convergence of the objective \eqref{eq:fair-kernel}, or \eqref{eq:minimax} for a finite class of ``top-level'' kernels $\kappa$ (as we use here), to the corresponding term based on asymptotic powers.
(Convergence of \eqref{eq:minimax} over an infinite class of $\kappa$ would also follow with a similar argument to that of \citeauthor{liu2020learning}.)

\section{Training Details}\label{app:training}
\begin{table}[h]
\centering
\resizebox{\columnwidth}{!}{%
\begin{tabular}{@{}c|c|cc|cc|c@{}}
\toprule
\multirow{2}{*}{Dataset} &  & \multicolumn{2}{c|}{\# Hidden Units} & \multicolumn{2}{c|}{Optimizer} &  \\ \cmidrule(l){2-7} 
                & Input Size & Encoder ($\phi$)              & Classifier ($g$) & Type     & Learning Rate & Batch Size \\ \midrule
Adult           & 114        & 256, 128, 64, 32, 16 & 16         & Adam     & 0.0001        & 64         \\
COMPAS          & 11         & 8, 8, 8              & 8          & Adadelta & 2.0           & 64         \\
Heritage Health & 65         & 256, 128, 64, 32, 16 & 16         & Adam     & 0.0001        & 64         \\ \bottomrule
\end{tabular}%
}
\caption{Network architectures, optimizers and batch sizes used to train our models. All layers are interspersed with Leaky ReLU activations.}
\label{tab:hyperparams}
\end{table}
\Cref{tab:hyperparams} contains the architecture and optimizer details used to train our models. All models were trained for a maximum of 100 epochs and we employed early stopping on the validation loss with a patience of 20 epochs. The encoder for CFAIR and LAFTR in the case for the Adult dataset contains 60 hidden units followed by 60 units in the classifier as described in their original papers as both these methods performed poorly with the architecture in \cref{tab:hyperparams}. All sIPM models were trained without the reconstruction loss term. To compute the power in \eqref{eq:block-power-est} we set $b = \sqrt{m}$ and therefore, $B=m/b=\sqrt{m}$.

\end{document}